\documentclass[12pt,twocolumn]{article}
\pdfoutput=1

\usepackage[utf8]{inputenc}
\usepackage[english]{babel}

\usepackage[letterpaper,hmargin={3cm,2cm},vmargin={2.5cm,2cm}]{geometry}
\usepackage{times}

\usepackage{standalone, tikz, pgf}

\usepackage{amsmath}
\usepackage{amssymb}
\usepackage{amsthm}
\usepackage{bm}

\newcommand{\mi}{\mathrm{i}}

\newcommand{\cis}{\mathrm{cis}}
\newcommand{\itan}{\mathrm{atan2}}

\newtheorem{definition}{Definition}
\newtheorem{proposition}{Proposition}

\title{Position Equations of a 3RPR Planar Manipulator}
\author{Sureyya Sahin\\
\texttt{sahin508@gmail.com}}
\date{}

\begin{document}
\maketitle
\begin{abstract}
\textit{We study parametric equations, which describe the position of an in-parallel planar manipulator. We discuss isometries in the Gauss plane, then we write the loop-closure equations in terms of the rotations as the parameters.}
\end{abstract}
\section{Introduction}
A 3RPR parallel manipulator, which is a closed-loop linkage, consists of a moving platform--a ternary link--and three simple kinematic chains, each composed of revolute, prismatic, and revolute joints, attached to vertices of the moving platform. The position equations are essential in mathematical modeling of the 3RPR parallel manipulators. 

In the literature, the position equations are generally obtained by means of parametrizations in the Euclidean plane. For example, Shigley and Uicker \cite{shigley_theory_1995} used loop-closure equations in obtaining the equations describing the position of closed-loop linkages in terms of trigonometric functions. As another example, Wampler \cite{wampler1999solving} studied parametric equations of closed-loop linkages in the Gauss plane. In this study, we obtain the position equations of the 3RPR parallel manipulator in the Gauss plane with a similar approach to that of Wampler \cite{wampler1999solving}, which he applied to a closed-loop mechanism in his study. However, we will concentrate more on geometry instead of mechanical motions; thus, we emphasize isometries and transformations, and use loop-closure equations directly.

The organization of this short paper is as follows: we discuss the geometric foundations for obtaining the position equations for the 3RPR in-parallel manipulator in Section \ref{sec:geom_foundations}; specifically, we discuss the complex numbers, isometries and transformations in the Gauss plane. In Section \ref{sec:pos_equations}, we derive the loop--closure equations of the manipulator.


\section{Complex Numbers}\label{sec:geom_foundations}
In this section, we study properties of complex numbers. Additionally, we discuss isometries in the Gauss plane.
\subsection{Complex Numbers and the Gauss Plane}
We define complex numbers, which represent points in the Cartesian plane, as a set $\mathbb{C}=\{x+\mi y:x,y\in\mathbb{R},\mi^2=-1\}$. Denoting the $x$-axis of the Cartesian plane as the real axis and $y$-axis as the imaginary axis, we obtain a different interpretation of the Cartesian plane, and we name it as the Gauss plane \cite{dodge_euclidean_2004}. 

We affiliate a point $Z$ in the Gauss plane by a complex number $z$ and name $x+\mi y$ as the rectangular form of $z$. Thus, the set of complex numbers $\mathbb{C}$, with identity elements for complex addition and multiplication being $0$ and $1$, forms a field. As an alternative to the rectangular form, for a given point, we define $r$ as the length of the line segment from the origin to the point (amplitude), and $\theta$ as the angle formed by the real axis and the line segment (argument). Thus, a point $Z$ in the Gauss plane can be expressed as $z=r\cis{(\theta)}$, where $\cis$ is defined in Definition \ref{def:cis}.
\begin{definition}\label{def:cis}
We let the set of points on a unit circle and its associated argument be denoted as $\mathbb{S}$ and $\theta$, respectively. Then, we define $\cis:\mathbb{S}\rightarrow \mathbb{C}$ such that $\cis{(\theta)}=\cos{(\theta)}+\mi \sin{(\theta)}$.
\end{definition}

In addition to using $\mathbb{C}$ for representing points in the Gauss plane, we can also express vectors by means of $\mathbb{C}$. Hence, a directed line from the origin at $0$ to the point $x+\mi y$ represents a vector. The orthonormal vectors coincident with the coordinate axis would be $\bm{1}$ and $\bm{i}$, respectively. Thus, any vector can be written as a linear combination of $\bm{1}$ and $\bm{i}$ as $\bm{z}=x\bm{1}+y\,\bm{i}$. We drop $\bm{1}$ from this expression \cite{dodge_euclidean_2004}; furthermore, since $\bm{i}$ has only imaginary component, we use $\mi$ in vector expressions instead of $\bm{i}$. Therefore, we express a vector $\bm{z}$ by using the standard representations of the complex numbers, so $\bm{z}=z$ in $\mathbb{C}$.

The conjugate of a complex number $z = x+\mi y\in\mathbb{C}$, which is written as $\overline{z} = x-\mi y$, can be viewed as an operation on $\mathbb{C}$. So, we provide Proposition \ref{prop:conjaddmul} \cite{pedoe_geometry_1988} on the conjugate of addition and multiplication of complex numbers.
\begin{proposition}\label{prop:conjaddmul}
 Let $z_1,z_2,\hdots,z_n\in\mathbb{C}$. Then,
 \begin{equation}
\overline{z_1+z_2+\cdots+z_n}=\overline{z}_1+\overline{z}_2+\cdots+\overline{
z}_n
\label{eq:additive_conj}
 \end{equation}
 and furthermore,
 \begin{equation}
\overline{z_1\,z_2\,\cdots\,z_n}=\overline{z}_1\,\overline{z}_2\,\cdots\,
\overline{z}_n
\label{eq:multip_conj}
 \end{equation}
\label{prop:3}
\end{proposition}
\begin{proof}
  We prove the equation \ref{eq:additive_conj} first. The case of $n=1$ is trivially satisfied. Then, we pick $n=2$ to obtain 
$\overline{z_1+z_2}=\overline{x_1+\mi\,y_1+x_2+\mi\,y_2}=x_1+x_2-\mi\,
(y_1+y_2)=x_1-\mi\,y_1+x_2-\mi\,y_2=\overline{z}_1+\overline{z}_2$. Thus, the 
equation holds for $n=2$. Now, we suppose that the equation is true for 
$n-1$. Therefore, $\overline{z_1+z_2+\cdots+z_{n-1}}=\overline{z}_1+\overline{z}
_2+\cdots+\overline{z}_{n-1}$. Hence, we can write
\begin{eqnarray*}
\overline{z_1+z_2+\cdots+z_n}&=&\overline{z_1+z_2+\cdots+z_{n-1}}\\
  {}&&+\overline{z}_n\\
&=&\overline{z}_1+\overline{z}_2+\cdots+\overline{z}_{n-1}\\
  {}&&+\overline{z}_n
\end{eqnarray*}
by using a similar argument to the case of $n=2$, and by substituting the expression for $n-1$. Therefore, we conclude the equation \ref{eq:additive_conj} 
holds in $\mathbb{C}$.

We prove the equation \ref{eq:multip_conj} similarly. Letting $n=1$, we observe that the equation becomes trivially $\overline{z}_1=\overline{z}_1$. 
Next, we let $n=2$. Thus, we write:
\begin{eqnarray*}
 \overline{z_1\,z_2}&=&\overline{(x_1+\mi\,y_1)(x_2+\mi\,y_2)}\\
 &=&(x_1x_2-y_1y_2)-\mi\,(x_1y_2+y_1x_2)\\
 &=&(x_1-\mi\,y_1)(x_2-\mi\,y_2)\\
 &=&\overline{z}_1\,\overline{z}_2
\end{eqnarray*}
Now, we suppose that the equation \ref{eq:multip_conj} holds for $n-1$. Thus, $\overline{z_1z_2\cdots 
z_{n-1}}= \overline{z}_1\overline{z}_2\cdots\overline{z}_{n-1}$. Therefore, for $n$, we write:
\begin{eqnarray*}
 \overline{z_1z_2\cdots z_n} &=& \overline{z_1z_2\cdots 
z_{n-1}}\,\overline{z}_n\\
&=&\overline{z}_1\overline{z}_2\cdots\overline{z}_{n-1}\overline{z}_n
\end{eqnarray*}
Thus, we conclude that the equation \ref{eq:multip_conj} holds.
\end{proof}

\subsection{Isometries in the Gauss Plane}\label{subsec:isomet}
The isometries with which we are interested are reflections, translations, and rotations. The propositions in this section, except for \ref{prop:homomorphism}, are available--without proofs--in \cite{dodge_euclidean_2004}. We introduce reflections in the Gauss plane by means of the complex conjugates. Since the reflection is a mirror about a line, we write:
\begin{proposition}
A reflection $z'$ of $z\in\mathbb{C}$ in the $x$-axis has the equation:
\begin{equation}
z'=\overline{z}
\end{equation}
\end{proposition} 
\begin{proof}
Because the reflection is in the $x$-axis, we have the real and imaginary components of $z$ in the mirror as $x'=x$ and $y'=-y$. Thus, $z'=x-\mi y=\overline{z}$.
\end{proof}

A translation shifts the objects in the Gauss plane by a vector $a\in\mathbb{C}$. For a given translation $a=h+\mi k$ arising from the origin, we can find the location of $z'$ initially at $z$. Consequently, we write:
\begin{proposition}
A translation of an object from $z$ to $z'$ by means of $a$ is:
\begin{equation}
  z'=z+a
\end{equation}
\end{proposition}
\begin{proof}
We suppose $a$ is the translation which acts on $z$. Then, the point associated with $z$ would shift to $z'$. Using the vectors related to $z'$ and $z$, we find that $z'$ is the diagonal of the parallelogram whose sides are $a$ and $z$. Therefore, we conclude that $z'= z+a$.
\end{proof}

We introduce the rotations by means of the unit circle centered in the origin of the Gauss plane. Consider a unit vector arising from the origin to a point on the unit circle, and has the angle $\theta$ as its argument. If we rotate this vector by an angle $\phi$, then we obtain the new position of the vector as $\cis{(\theta+\phi)}$ using the Definition \ref{def:cis}. Hence, the function $\cis$ is used to express rotations in the Gauss plane. Thus, we introduce Proposition \ref{prop:rot} for expressing rotations.
\begin{proposition}\label{prop:rot}
A rotation of an object $z$ about a point $A$ by an angle $\theta$, which results in the vector $z'$, has the following representation:
\begin{equation}
  z'-a=\cis(\theta)(z-a)
\end{equation}
\end{proposition}
\begin{proof}
We let $z'$ be the new vector as a result of the rotation, and $z$ the original vector. Additionally, we let $a$ be the vector associated with the point $A$. Then, we write $r'=z'-a$ and $r=z-a$. Thus, using the rotation function $\cis$, we obtain $r'=\cis(\theta)r$. Therefore, substituting back for the difference vectors $r$ and $r'$, we obtain $z'-a=\cis(\theta)(z-a)$.	
\end{proof}

We observe that the angular quantities are additive; that is, the angles $\theta_i$ associated with successive rotations can be added to obtain the resulting angle as $\theta=\theta_1+\cdots+\theta_n$. We would like to obtain a similar relation to obtain a resulting rotation for given successive rotations. Thus, we provide the following proposition:
\begin{proposition}\label{prop:homomorphism}
The function $\cis :\mathbb{R}\rightarrow\mathbb{S}$ is a homomorphism between $\mathbb{R}$ with the addition and the $\mathbb{S}$ with multiplication operations.
\end{proposition}
\begin{proof}
We claim that:
\begin{equation*}
\cis(\theta_1+\theta_2+\cdots+\theta_n)=\cis(\theta_1)\cdot\cis(\theta_2)\cdots\cis(\theta_n)
\end{equation*}
We pick $\theta_1,\theta_2\in\mathbb{R}$. Then, we write
\begin{eqnarray*}
\cis(\theta_1+\theta_2)&=&\cos(\theta_1+\theta_2)+\mi\sin(\theta_1+\theta_2)\\
& = & \cos(\theta_1)(\cos(\theta_2)+\mi \sin(\theta_2))\\
{}&&+\mi^2\sin(\theta_1)(\sin(\theta_2)\\
{}&&-\mi \cos(\theta_2))\\
& = &\cis(\theta_1)\cdot\cis(\theta_2)
\end{eqnarray*}
Now, we suppose that 
\begin{eqnarray*}
\cis(\theta_1+\theta_2+\cdots+\theta_{n-1}) &=& \cis(\theta_1)\cdot\cis(\theta_2)\\
{}&&\cdots\cis(\theta_{n-1})
\end{eqnarray*}
So, we can write
\begin{eqnarray*}
\cis(\theta_1+\cdots+\theta_n)& = &\cis(\theta_1+\cdots+\theta_{n-1})\\
{}&&\cdot\cis(\theta_n)\\
& = & \cis(\theta_1)\cdot\cdots\cdot\cis(\theta_{n-1})\\
{}&&\cdot\cis(\theta_n)
\end{eqnarray*}
Therefore, we conclude that $\cis$ is a homomorphism between $\mathbb{R}$ with addition, and $\mathbb{C}$ with multiplication.
\end{proof}
We can use $\cis$ and $\overline{\cis}$ as a pair to express the unit circle $\mathbb{S}$. Consequently, we write:
\begin{proposition}\label{prop:10}
We let $\theta\in\mathbb{R}$. The unit circle $\mathbb{S}$ can be represented as $\cis(\theta)\overline{\cis}(\theta)=1$.
\end{proposition}
\begin{proof}
Pick $\theta\in\mathbb{R}$. Then, $\cis(\theta)\cdot\overline{\cis}(\theta)=(\cos(\theta)+\mi\sin(\theta))\cdot(\cos(\theta)-\mi \sin(\theta))=\cos^2(\theta)+\sin^2(\theta)=1$. Therefore, $\cis(\theta)\cdot\overline{\cis}(\theta)=1$ represents the unit circle.
\end{proof}

\begin{figure}[!t]
 \begin{center}
   \begin{tikzpicture}[scale=3]
 \node (C2) {$\mathbb{C}^2$};
  \node (R2) [below right of=C2, node distance=2cm] {$\mathbb{R}^2$};
  \node (P) [above right of=R2,node distance=2cm] {$\mathbb{S}$};
  \draw[->] (C2) to node [left] {$(x,y)$} (R2);
  \draw[->] (R2) to node [right] {$\alpha$} (P);
  \draw[->] (C2) to node [above] {$\theta$} (P);
\end{tikzpicture}
 \end{center}
 \caption{Commutative diagram between $\mathbb{S}$ and product sets $\mathbb{C}$, $\mathbb{R}^2$.}
 \label{fig:commutative}
\end{figure}
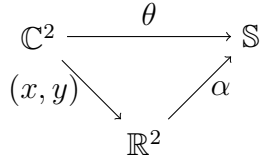
The pair $(\cis,\overline{\cis})\in\mathbb{S}^2\subset\mathbb{C}^2$ is used in the parametrization of position equations. Specifically, we would like to obtain relations between the unit circle $\mathbb{S}$, complex numbers $\mathbb{C}^2$ and the Cartesian coordinates in $\mathbb{R}^2$. Thus, we define $(x,y):\mathbb{C}^2\rightarrow\mathbb{R}^2$ with 
$x(\cis,\overline{\cis})=(\cis+\overline{\cis})/2$, $y(\cis,\overline{\cis}) = 
(-\mi(\cis-\overline{\cis}))/2$. Similarly, we define 
$\alpha:\mathbb{R}^2\rightarrow \mathbb{S}$ as $\alpha = \itan{(y,x)}$. Hence, we draw the Fig. \ref{fig:commutative} to obtain the function $\theta:\mathbb{C}^2\rightarrow\mathbb{S}$. Consequently, $\theta=\alpha\circ(x,y)$, and therefore 
\begin{equation}\label{eq:theta}
\theta=\itan\left(-\mi(\cis-\overline{\cis}),(\cis+\overline{\cis})\right)
\end{equation}


\section{Position Equations}\label{sec:pos_equations}
The equations describing the posture of the in-parallel planar manipulator, which has three degrees of freedom, can be obtained by means of loop-closure equations. 

As an initial step, we would like to obtain the number of loops, which determines the quantity of the loop-closure equations, for this manipulator. A general formula for calculating the number of loops is given by Waldron and Kinzel \cite{waldron_kinzel_kinematics_dynamics_design_machinery_2004} as $L=j+1-n$, with $j$ being the total number of joints, and $n$ being the total number of links. We apply this formula to the 3RPR manipulator in Fig. \ref{fig:rpr-param}; we observe that $j=9$, $n=8$. Thus, we calculate $L=9+1-8=2$ loops; therefore, we can write two independent loop-closure equations for this manipulator. For the 3RPR manipulator shown in Fig. \ref{fig:rpr-param}, we let $s_a=|AO_a|$, $s_b=|BO_b|$, and $s_c=|CO_c|$ be the lengths of each connectors. We denote the unknown joint variables by $\theta_a$, $\theta_b$, $\theta_c$, and $\alpha$. 
\begin{figure}[!t]
 \begin{center}
  \begin{tikzpicture}
[scale=0.75]
\draw(0,0)--(0.70711,0.70711);
\draw(0.70711,0.70711)--++(135:0.2)--++(45:1);
\draw(0.70711,0.70711)--++(-45:0.2)--++(45:1);
\draw(45:1.2)--(2.61805,2.61805);
\draw(45:1.2)--++(135:0.1);
\draw(45:1.2)--++(-45:0.1);
\draw[black!20,fill] (2.61805,2.61805) -- (5.29389,3.97447) -- (1.5,5.40192) -- (2.61805,2.61805);
\draw(2.61805,2.61805)--(5.29389,3.97447);
\draw(2.61805,2.61805)--(1.5,5.40192);
\draw[fill=white] (0,0) circle (0.2) node (ja1) {};
\draw (0.4,-0.4) -- (0.4,0) arc (0:180:.4) (-0.4,0) -- (-0.4,-0.4) -- 
(0.4,-0.4) node [left=20pt] {$A$};
  \foreach \x in {-0.4, -0.2, 0, 0.2, 0.4}
    \draw (\x,-0.4) -- (\x-0.1,-0.5);
\draw[fill=white] (2.61805,2.61805) circle (0.2) node [below right=0pt] (ja3) {$O_a$};
\draw(6,0)--(5.82508,0.98458);
\draw(5.82508,0.98458)--++(190:0.2)--++(100:1);
\draw(5.82508,0.98458)--++(10:0.2)--++(100:1);
\draw(5.79009,1.1815)--(5.29389,3.97447);
\draw(5.79009,1.1815)--++(190:0.1);
\draw(5.79009,1.1815)--++(10:0.1);
\begin{scope}[shift = {(6, 0)}]
   \draw[fill=white] (0,0) circle (0.2) node (jb1) {};
  \draw (0.4,-0.4) -- (0.4,0) arc (0:180:.4) (-0.4,0) -- (-0.4,-0.4) -- 
(0.4,-0.4) node [left=20pt] {$B$};
  \foreach \x in {-0.4, -0.2, 0, 0.2, 0.4}
    \draw (\x,-0.4) -- (\x-0.1,-0.5);
  \end{scope}
\draw[fill=white] (5.29389,3.97447) circle (0.2) node [right=4pt](jb3){$O_b$};
\draw (3,8) -- (2.5,7.13397);
\draw (2.5,7.13397)--++(-30:0.2)--++(-120:1);
\draw (2.5,7.13397)--++(150:0.2)--++(240:1);
\draw (2.4,6.96077)--(1.5,5.40192);
\draw (2.4,6.96077)--++(-30:0.1);
\draw (2.4,6.96077)--++(150:0.1);
  \begin{scope}[shift = {(3,8)}, rotate = 180]
   \draw[fill=white] (0,0) circle (0.2) node (j13) {};
  \draw (0.4,-0.4) -- (0.4,0) arc (0:180:.4) (-0.4,0) -- (-0.4,-0.4) -- 
(0.4,-0.4) node [below right=14pt] {$C$};
  \foreach \x in {-0.4, -0.2, 0, 0.2, 0.4}
    \draw (\x,-0.4) -- (\x-0.1,-0.5);
  \end{scope}
\draw[fill=white] (1.5,5.40192) circle (0.2) node [left=4pt] (jc3) {$O_c$};
\draw[->, very thin] (-0.05,0) -- (2.5,0) node [below] {$\mathrm{Re}$};
\draw[->, very thin] (0,-0.05) -- (0,2.5) node [left] {$\mathrm{Im}$};
\draw [->] (0.75,0) arc (0:45:0.75) node [below right=3pt] (ta1) {};
\node [right=-2pt] at (ta1) {$\theta_a$};
\draw [->] (6.75,0) arc (0:100:0.75) node [right=14pt] (tb1) {$\theta_b$};
\draw [very thin] (6,0) -- (7,0);
\draw [->] (3.75,8) arc (0:240:0.75) node [above left=6pt] (tc1) {$\theta_c$};
\draw [very thin] (3,8) -- (4,8);
\draw [->] (3.36805,2.61805) arc (0:27:0.75) node [below right=1pt] (alpha) {};
\node [right=-3pt] at (alpha) {$\alpha$};
\draw [very thin] (2.61805,2.61805) -- (3.61805,2.61805);
\draw (3.28701,2.95715) arc (27:112:0.75) node [above right=3pt] (beta) {};
\node[right=2pt] at (beta) {$\beta$};
\end{tikzpicture}
 \end{center}
 \caption{The 3RPR in-parallel planar manipulator. The moving platform is in gray color.}
 \label{fig:rpr-param}
\end{figure}
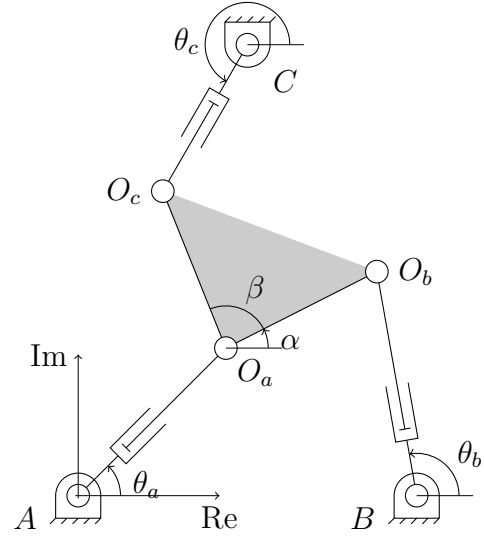

The first loop that we will consider is through the joints at $A$, $O_a$, $O_b$, and $B$. We let $d_{ab}$ be the vector arising from the origin at $A$ to the point $B$, $o_a$ and $o_b$ be the vectors originating from $A$, which are associated with the points $O_a$ and $O_b$. We would like to obtain the loop-closure equation in terms of $\cis$, $\overline{\cis}$, and the prismatic joint variable $s$. Thus, we pick a line $AA'$, which connects the point $A$ at the origin of the coordinate frame to the point $A'$ at center of the chamber of the prismatic joint as in Fig. \ref{fig:rpr-param}, collinear with the real axis of the coordinate frame.  Then, we perform a rotation $AA'$ by an angle $\theta_a$, followed by a translation by an amount of $A'O_a$ along the rotated line. This motion sequence results in the position at $O_a$, the associated vector of which can be expressed as $(|AA'|+|A'O_a|)\cis(\theta_a)$; hence, $|AO_a|\cis(\theta_a)$. Therefore, we can write $o_a=s_a\cis(\theta_a)$. Next, we consider the rotation which transforms the side of the moving platform by an angle $\alpha$. Using the Proposition \ref{prop:rot}, we write $o_b-o_a=\cis(\alpha)(o_b'-o_a)$, where $o_b'$ is the vector associated with $O_b$ before the rotation. But $o_{ab}=o_b'-o_a$; thus, $o_b=o_a+\cis(\alpha)l_{ab}$, where $l_{ab}=|o_{ab}|=|O_aO_b|$. As a next step, we draw a line parallel to the real axis and passing through $O_b$; thus, from the transversal passing through $B$ and $O_b$, we find the angle associated with the joint at $O_b$ as $\theta_b+\pi$. Then, denoting endpoint of the piston of the prismatic joint as $B'$ in the Figure \ref{fig:rpr-param}, we calculate the position of the point $B'$, which is a result of rotation about $O_b$, as $b'-o_b=\cis(\theta_b+\pi)|O_bB'|$. Furthermore, we write the translation from $B'$ to $B$ as $b=b'+\cis(\theta_b+\pi)|B'B|$. Eliminating $b'$ and substituting $s_b=|O_bB'|+|B'B|$, we obtain $b=o_b+\cis(\theta_b+\pi)s_b$. Then, using the Proposition \ref{prop:homomorphism}, we write $b=o_b+\cis(\theta_b)\cis(\pi)s_b=o_b-\cis(\theta_b)s_b$. As a final step for closing the loop, we translate $b$ by $d_{ba}$ to obtain $a=b+d_{ba}$; additionally, noting $d_{ba}=-d_{ab}$, and $a=0$ for $A$ being the origin of the fixed frame, we write $0=b-d_{ab}$. Back substituting $b$ into this equation, we write $o_b-\cis(\theta_b)s_b-d_{ab}=0$. Then, eliminating $o_b$, we write $o_a+\cis(\alpha)l_{ab}-\cis(\theta_b)s_b-d_{ab}=0$. Finally, we eliminate $o_a$ to obtain:
\begin{equation}\label{eq:lce1}
s_a \cis_a+l_{ab} \cis_\alpha-s_b \cis_b-d_{ab}=0
\end{equation}
where we abbreviate $\cis(\theta_a)$, $\cis(\alpha)$, $\cis(\theta_b)$ with $\cis_a$, $\cis_\alpha$, and $\cis_b$ respectively. Thus, Eq. \ref{eq:lce1} is the first loop-closure equation.

For the second loop-closure equation, we use the expression for $o_a$ as found in the derivation of the Eq. \ref{eq:lce1}. Denoting $l_{ac}=|O_aO_c|$, we use the Proposition \ref{prop:rot} to obtain $o_c=o_a+\cis(\alpha+\beta)l_{ac}$. Then, by the Proposition \ref{prop:homomorphism}, $o_c=o_a+\cis(\alpha)\cis(\beta)l_{ac}$. As a next step, we draw lines parallel to the real axis of the coordinate system passing through points $O_c$ and $C$. From a transversal through $CO_c$, we find the rotational angle associated with the joint at $O_c$ as $\theta_c-\pi$. We denote the end point of the piston as $C'$ as in Fig. \ref{fig:rpr-param}; thus, we obtain $c'=o_c+\cis(\theta_c-\pi)|O_cC'|$. Using the Proposition \ref{prop:homomorphism}, we write $c'=o_c+\cis(\theta_c)\cis(-\pi)|O_cC'|$; therefore, $c'=o_c-\cis(\theta_c)|O_cC'|$. Additionally, the translation of $c'$ to $c$ can be represented as $c=c'+\cis(\theta_c-\pi)|C'C|=c'-\cis(\theta_c)|C'C|$. Noting $s_c=|O_cC'|+|C'C|$, we eliminate $c'$ to obtain $c=o_c-\cis(\theta_c)s_c$. To calculate $a$, we translate $c$ by an amount $d_{ca}=-d_{ac}$ so that $a=c-d_{ac}$. Using $a=0$ and substituting for $c$, we write $o_c-\cis(\theta_c)s_c-d_{ac}=0$. Then, we eliminate $o_c$ to obtain $o_a+\cis(\alpha)\cis(\beta)l_{ac}-\cis(\theta_c)s_c-d_{ac}=0$. We substitute $o_a$ to write Eq. \ref{eq:lce2} as the second loop-closure equation.
\begin{equation}\label{eq:lce2}
s_a\cis_a+l_{ac}\cis_\beta\cis_\alpha-s_c\cis_c-d_{ac}=0
\end{equation}
We use $\cis_c$ for $\cis(\theta_c)$ and $\cis_\beta$ for $\cis(\beta)$ for convenience. We seek parametrization based on $\cis$ and $\overline{\cis}$ representations of the unit circle in $\mathbb{C}^2$. Using the equations \ref{eq:additive_conj} and \ref{eq:multip_conj}, we take the conjugates of the equations \ref{eq:lce1} and \ref{eq:lce2}. Thus, we obtain the conjugates given by the equations \ref{eq:colce1} and \ref{eq:colce2}. 
\begin{equation}\label{eq:colce1}
s_a\overline{\cis}_a+l_{ab}\overline{\cis}_\alpha-s_b\overline{\cis}_b-\overline{d}_{ab}=0
\end{equation}
\begin{equation}\label{eq:colce2}
s_a\overline{\cis}_a+l_{ac}\overline{\cis}_\beta\overline{\cis}_\alpha-s_c\overline{\cis}_c-\overline{d}_{ac}=0
\end{equation}


\section{Conclusion}
This short paper is about obtaining the equations describing the position of a 3RPR in-parallel planar manipulator in the Gauss plane. We utilize isometries in writing loop-closure equations. The resulting equations with their complex conjugates constitute four polynomial equations. 

We note that parametric position equations of other in-parallel planar manipulators can be obtained with a similar way to the method in this paper. We believe the position equations can be used in analysis and control of in-parallel planar manipulators. 


\bibliography{sahin-brief}
\end{document}